\documentclass{article}
\usepackage[preprint, nonatbib]{neurips_2019}
\usepackage[utf8]{inputenc} 
\usepackage[T1]{fontenc}    
\usepackage{hyperref}       
\usepackage{url}            
\usepackage{booktabs}       
\usepackage{amsfonts}       
\usepackage{nicefrac}       
\usepackage{microtype}      
\usepackage{amssymb}
\usepackage{amsthm}
\usepackage{color}
\usepackage{subcaption}
\usepackage{algorithm}
\usepackage[noend]{algorithmic}
\usepackage{multirow}
\usepackage{bbm}
\usepackage{enumitem}
\usepackage{multirow}
\usepackage{amsmath}
\usepackage{bbm}
\usepackage{wrapfig}
\usepackage[square,sort,comma,numbers]{natbib}

\newtheorem{obv}{Observation}[section]

\title{Graph Neural Reasoning May Fail in Certifying Boolean Unsatisfiability}

\author{
	Ziliang Chen\footnotemark[1] \\
	Department of Computer Science \\
	Sun Yat-sen University \\
	GuangZhou, Guangdong, China \\
	c.ziliang@yahoo.com \\
	\And 
	Zhanfu Yang\footnotemark[1]\thanks{indicates alphabetic order.}\\
	Department of Computer Science\\
	Purdue University\\
	West Lafayette, IN, USA \\
	yang1676@purdue.edu \\
}
 
\begin{document}
	
	\maketitle
	
	\begin{abstract}
		It is feasible and practically-valuable to bridge the characteristics between graph neural networks (GNNs) 
		and logical reasoning. Despite considerable efforts and successes witnessed to solve 
		Boolean satisfiability (SAT), 
		it remains a mystery of GNN-based solvers for more complex \emph{predicate logic} formulae.
		In this work, we conjectures with some evidences, that generally-defined GNNs present several limitations to certify the unsatisfiability (UNSAT) in Boolean formulae. It implies that GNNs may probably fail in learning the logical reasoning tasks if they contain proving UNSAT as the sub-problem included by most predicate logic formulae. 
	\end{abstract}
	\section{Introduction}
	Logical reasoning problems span from simple propositional logic to complex predicate logic and high-order logic, with known theoretical complexities from NP-completeness \citep{Cook:1971:CTP:800157.805047} to semi-decidable and undecidable \citep{DBLP:journals/jsyml/Church36}.
	Testing the ability and limitation of machine learning tools on logical reasoning problems leads to a fundamental understanding of the boundary of learnability and robust AI, helping to address interesting questions in decision procedures in logic, program analysis, and verification as defined in the programming language community.
	
	There have been arrays of successes in learning propositional logic reasoning
	\citep{DBLP:conf/iclr/SelsamLBLMD19, amizadeh2018learning},
	which focus on Boolean satisfiability (\textbf{SAT}) problems as defined below. A Boolean logic formula is an expression composed of Boolean constants  ($\top$ : true, $\bot$ : false)
	, Boolean variables ($x_i$), and propositional connectives such as $\land$, $\lor$, $\lnot$ (for example $(x_1 \lor \neg x_2) \land (\neg x_1 \lor x_2)$). The SAT problem asks if a given Boolean formula can be satisfied (evaluated to $\top$) by assigning proper Boolean 
	values to the literal variables. A crucial feature of the logical reasoning domain (as is visible in the SAT problem) is that the inputs are often structural, where logical connections between entities (variables in SAT problems) are the key information. 
	
	SAT and its variant problems are almost NP-complete or even more complicated in the complexity. The fact motivates the emergence of sub-optimal heuristic that trades off the solver performance to rapid reasoning. In terms of the fast inference process, deep learning models are favored as learnable heuristic solvers \cite{amizadeh2018learning,DBLP:conf/iclr/SelsamLBLMD19,DBLP:journals/corr/abs-1904-12084}. Among them Graph Neural Networks (GNNs) have grasped amount of attentions, since the message-passing process delivers the transparency to interpret the inference within GNNs, thus, revealing the black box behind neural logical reasoning in the failure instances.
	
	However, it should be noticed that logical decision procedures
	is more complex that just reading the formulas correctly.
	It is unclear if GNN embeddings (from simple message-passing)
	contain all the information needed to
	reason about complex logical questions on top of the graph 
	structures derived from the formulas, 
	or whether the complex embedding schemes can be
	learned from backpropagation.
	Previous successes on SAT problems argued for the power of GNN, which can handle NP-complete problems \citep{DBLP:conf/iclr/SelsamLBLMD19, amizadeh2018learning}, whereas no evidences have been reported for solving semi-decidable predicate logic problems via GNN. The significant difficulty to prove the problems is the requirement of comprehensive reasoning over a search space, since a complete proof includes SAT and UNSAT (\emph{i.e.}, Boolean unsatisfiability). 
	
	Perhaps disappointingly, this work presents some theoretical evidences that support a pessimistic conjecture: GNNs do not simulate the complete solver for UNSAT. Specifically, we discover that the neural reasoning procedure learned by GNNs does simulate the algorithms that may allow a CNF formula changing over iterations. Those complete SAT-solvers, \emph{e.g.}, DPLL and CDCL, are almost common in the operation that adaptively alters the original Boolean formula that eases the reasoning process. So GNNs do not learn to simulate their behaviors. Instead, we prove that by appropriately defining a specific structure of GNN that a parametrized GNN may learn, the local search heuristic in WalkSAT can be simulated by GNN. Towards these results, we believe that GNN can not solve UNSAT in existing logical reasoning problems. 
	
	

	\section{Embedding Logic Formulae by GNNs}\label{sec:gnn_emb}

	\paragraph{Preliminary: Graph Neural Networks (GNNs).}\hspace{-0.6em} GNNs refer to the neural architectures devised to learn the embeddings of nodes and graphs via message-passing.
	Resembling the generic definition in \cite{DBLP:conf/iclr/XuHLJ19}, they consist of two successive operators to propagate the messages and evolve the embeddings over iterations:
	\begin{small}\begin{equation}\label{gnn}
		\begin{aligned}
		m^{(k)}_{v}\hspace{-0.3em}=\hspace{-0.3em}{\rm Aggregate}^{(k)}\Big(\{h^{(k-1)}_{u}:u\in \mathcal{N}(v)\}\Big), \ \ \ h^{(k)}_{v}\hspace{-0.3em}=\hspace{-0.3em}{\rm Combine}^{(k)}\Big(h^{(k-1)}_{v},m^{(k)}_{v}\Big)
		\end{aligned}
		\end{equation}\end{small}where $h^{(k)}_{v}$ denotes the hidden state (embedding) of node $v$ 
	in the $k^{th}$ iteration, and $\mathcal{N}(v)$ 
	denotes the neighbors of node $v$. In each iteration, the ${\rm Aggregate}^{(k)}(\cdot)$ aggregates hidden 
	states from node $v$'s neighbors $\{h^{(k-1)}_{u}:u\in \mathcal{N}(v)\}$ to produce the new message (\emph{i.e.}, $m^{(k)}_{v}$) for node $v$;
	${\rm Combine}^{(k)}(\cdot, \cdot)$ updates the embedding of $v$ in terms of its previous state and its current message. After a specific number of iterations (\emph{e.g.}, $K$ in our discussion), the embeddings should capture the global relational information of the nodes, which can be fed into other neural network modules for specific tasks.

	Significant successes about GNNs have been witnessed in relational reasoning \citep{DBLP:journals/corr/abs-1807-09244,DBLP:journals/corr/LiangSFLY16,DBLP:journals/corr/abs-1810-02338}, where an instance could be departed into multiple objects then encoded by a series of features with their relation. It typically suits representation in Eq. \ref{gnn}. Whereas in logical reasoning, a Boolean formula is in Conjunctive Normal Form (CNF) that consists of literal and clause items. In term of the independence among literals in CNF (so do clauses), \citep{DBLP:conf/iclr/SelsamLBLMD19} embeds a formula into a \emph{bipartite graph}, where the nodes denote the clauses and literals that are disjoint, respectively. In this principle, given a literal $v$ as a node, all the nodes of clauses that contains the literal are routinely treated as $v$'s neighbors, vice and versa for the node of each clause. We assume $\Phi$ is a logic formula in CNF, \emph{i.e.}, a set of clauses, and $\Psi(v) \in \Phi $ denote one of clauses within the logic formula $\Phi$ that contains literal $v$. Derived from Eq. \ref{gnn} , GNNs for logical reasoning can be further specified by
	
	\begin{small}\begin{equation}\label{bgnn}
		\begin{aligned}
		m^{(k)}_{v}\hspace{-0.3em}&=\hspace{-0.3em}{\rm Aggregate}^{(k)}_{L}\Big(\{h^{(k-1)}_{\Psi(v)}: \Psi(v)\in \Phi\}\Big),& h^{(k)}_{v}\hspace{-0.3em}&=\hspace{-0.3em}{\rm Combine}^{(k)}_{L}\Big(h^{(k-1)}_{v},h^{(k-1)}_{\neg v},m^{(k)}_{v}\Big), &s.t. \ \forall v\in L\\
		m^{(k)}_{\Psi(v)}\hspace{-0.3em}&=\hspace{-0.3em}{\rm Aggregate}_{C}^{(k)}\Big(\{h^{(k-1)}_{u}: u\in \Psi(v)\}\Big),& h^{(k)}_{\Psi(v)}\hspace{-0.3em}&=\hspace{-0.3em}{\rm Combine}_{C}^{(k)}\Big(h^{(k-1)}_{\Psi(v)},m^{(k)}_{\Psi(v)}\Big), &s.t. \ \forall \Psi(v)\in \Phi
		\end{aligned}
		\end{equation}\end{small}
	where $h^{(k)}_{v}$ and $h^{(k)}_{\Psi(v)}$ denote embeddings of the literal $v$ and the clause $\Psi(v)$ in the $k^{th}$ iteration ($h^{(k)}_{\neg v}$  denotes the embedding of the negation of $v$); $m^{(k)}_{v}$ and $m^{(k)}_{\Psi(v)}$ refer to their propagated messages. Since the value of a Boolean formula is determined by the value assignment of the literal variables, Eq. \ref{bgnn} solely requires the final-state literal embeddings $\{h^{(K)}_{u}, u\in L\}$ to predict the logical reasoning result.
	More specifically, we use $L$ and $C$ to denote a literal set and a clause set ($L$ and $C$ may be different for each CNF formula), then $\Psi(v)$ is a clause and $\Psi(v)$ denotes a clause including the literal $v\in L$.
	
	Note that the graph embeddings for SAT \cite{10.1007/978-3-319-40970-2_27} and 2QBF \cite{10.1007/978-3-319-40970-2_27} are generally represented by Eq.\ref{bgnn} . Hence our further analysis is based on Eq.\ref{bgnn} .
	
	\section{Certifying UNSAT by GNNs may Fail} Although existing researches showed that GNN can learn a well-performed solver for satisfiability problems, GNN-based SAT solvers actually have terrible performances in predicting
	unsatisfiability with high confidence \citep{DBLP:conf/iclr/SelsamLBLMD19} in a SAT formula, if the formula does not have a small unsatisfiable core (minimal number of clauses that is enough
	to cause unsatisfiability). In fact, some previous work \citep{amizadeh2018learning} even completely removed
	unsatisfiable formulas from the training
	dataset, since they slowed down the whole training process. 
	
	The difficulty in proving unsatisfiability is understandable, since constructing a proof of
	unsatisfiability demands a complete reasoning in the search space, which is more complex than
	constructing a proof of satisfiability that only requires a witness. Traditionally it relies on the recursive decision procedures that either traverse all possible assignments to construct the proof
	(DPLL \citep{DBLP:journals/cacm/DavisLL62}),
	or generate extra constraints from assignment trials that lead to conflicts, until some of the
	constraints contradict each other (CDCL \citep{DBLP:series/faia/SilvaLM09}). The line of recursive algorithms include some operation branches that reconfigure the bipartite graph behind the CNF in each step while they search. In the terms of a graph that may iteratively change (\emph{e.g.}, DPLL), perhaps miserably, their recursive processes can not be simulated by GNNs.
	
	\begin{obv}
		Given a recursive algorithm that iteratively reconfigures the graph, GNNs in Eq.2 can not simulate this recursive process. 
	\end{obv} 
	
	\begin{proof} Associating the aggregate and combine functions in Eq. \ref{bgnn}, we obtain the iterative update rule for the embedding of a literal $v$:
		
		\begin{small}
			\begin{equation}
			\begin{aligned}
			h^{(k)}_{v}\hspace{-0.3em}&=\hspace{-0.3em}{\rm Combine}^{(k)}_{L}\Big(h^{(k-1)}_{v},h^{(k-1)}_{\neg v},\hspace{-0em}{\rm Aggregate}^{(k)}_{L}\big(\{h^{(k-1)}_{\Psi(v)}: \Psi(v)\in \Phi\}\big)\Big)\\
			&={\rm Update}^{(k)}_{L}\Big(h^{(k-1)}_{v},h^{(k-1)}_{\neg v},\{h^{(k-1)}_{\Psi(v)}: \Psi(v)\in \Phi\}\Big), \ \ \ \ \ \ \ \ \ \ \ \ \ \ \ \ \ \ \ \ \ \ \ \ \ \ \  s.t. \ v\in L\label{lupdate}
			\end{aligned}
			\end{equation}\end{small}
		
		Towards this principle, we observe that the embedding update of $v$ in the current stage relies on the last-stage embeddings of $v$ and its negation $\neg v$, and the embeddings of all the clauses that include $v$ in a CNF formula ($\Psi(v)\in \Phi$). The literal $v$, $\neg v$ and the clauses containing $v$ are consistent over iterations. Hence if the update function (Eq. \ref{lupdate}) is consistent over the iterations in Eq.2, \emph{i.e.}, $\forall k\in \mathbb{N}_{+}$, ${\rm Update}^{(k)}_{L}={\rm Update}_{L}$, where ${\rm Update}_{L}$ means the update for literal embedding, GNNs derived from Eq. \ref{lupdate} receive a fixed graph generated by a CNF formula as input.
		However, if a recursive algorithm iteratively changes the graph that represents a CNF formula, it implies that there must be a clause that was changed (or eliminated) after this iteration, since clauses are permutation-invariant in a CNF formula. Accordingly there must be a literal embedding whose update process depends on a clause different from the previous iteration. It contradicts the literal embedding update function learned by Eq. \ref{lupdate} with $\forall k\in \mathbb{N}_{+}$, ${\rm Update}^{(k)}_{L}={\rm Update}_{L}$.
	\end{proof}
	
	Hence the message-passing in GNNs could not resemble the procedures in the complete SAT-solvers.
	In fact, GNNs are rather similar to learning a subfamily of \emph{incomplete} SAT solvers (GSAT, WalkSAT 
	\citep{DBLP:conf/dimacs/SelmanKC93}), which randomly assign variables
	and stochastically search for local witnesses. 
	
	\begin{obv}\label{observation2}
		GNNs in Eq. \ref{bgnn} may simulate the local search in WalkSAT. 
	\end{obv} 
	\begin{proof}
		Recall the iterative update routine of WalkSAT: starting by assigning a random value to each literal variable in a formula, it randomly chooses an unsatisfied clause in the formula and flips the value of a Boolean variable within that clause. Such process is repeated till the literal assignment satisfies all clauses in the formula. Here we \emph{construct the optimal aggregation and combine functions derived from Eq. \ref{bgnn} }, which are designed to simulate the procedure of WalkSAT. In this way, if the aggregation and combine functions in Eq. \ref{bgnn} approximate these optimal aggregation and combine functions, the GNN may simulate the local search in WalkSAT.
		
		Given a universe of literals in logical reasoning, we first initiate the embeddings of them and their negation, thus, $\forall v\in L$, random value of $h^{(0)}_{v}$ and $h^{(0)}_{\neg v}$ are initiated. This assignment can be treated as the Boolean value that belong to different literals, which have been mapped from a binary vector into a real-value embedding space about the literals. We also randomly initiate the clause embeddings $h^{(0)}_{\Psi(v)}$ for reasoning each formula that contains  the clause $\Psi(v)$. Here we define the \emph{optimal} aggregation and combine functions that encode literals and clauses respectively, which GNNs in Eq. \ref{bgnn} may learn if they attempt to simulate WalkSAT:
		\begin{small}    \begin{equation}
			\begin{aligned}\label{oal}
			{m}^{(k)}_{v}\hspace{-0.1em}&=\hspace{-0.1em}\overline{\rm Aggregate}_{L}\Big(\{h^{(k-1)}_{\Psi(v)}: \Psi(v)\in \Phi\}\Big),\\
			&=\left\{
			\begin{aligned}
			\boldsymbol{\epsilon}^{(k)} &, & \prod_{\Psi(v)}||h^{(k-1)}_{\Psi(v)}||= 0  \\
			\boldsymbol{0} &, & \prod_{\Psi(v)}||h^{(k-1)}_{\Psi(v)}||\neq 0 
			\end{aligned}
			\right.  \ \ \ \ \ \ s.t. \ \forall v\in L
			\end{aligned}
			\end{equation}\end{small}where $\overline{\rm Aggregate}_{L}\hspace{-0.1em}(\cdot)$ denotes the optimal aggregation function to propagate literal messages and  ${m}^{(k)}_{v}$ denotes the optimally propagated message of literal $v$ in the $k$ iteration; $\boldsymbol{0}$ is a zero-value vector; $\boldsymbol{\epsilon}^{(k)}$ denotes a bounded non-zero random vector generated in the $k$ iteration; $||\cdot||$ indicates a vector norm.  
		\begin{equation}
		\begin{aligned}\label{ocl}
		h^{(k)}_{v}\hspace{-0.3em}&=\hspace{-0.3em}\overline{\rm Combine}_{L}\Big(h^{(k-1)}_{v},h^{(k-1)}_{\neg v},{m}^{(k)}_{v}\Big)\\
		&=\left\{
		\begin{aligned}
		h^{(k-1)}_{\neg v} &, & v \hspace{-0.3em}=\hspace{-0.3em} \underset{\forall u\in L}{\arg\max}\{||{m}^{(k)}_{u}||\} \ {\rm and} \ ||{m}^{(k)}_{v}||>0 \\
		h^{(k-1)}_{v} &, &{\rm otherwise} 
		\end{aligned}
		\right.\ \ \ \ \ \ s.t. \ \forall v\in L
		\end{aligned}
		\end{equation}where $\overline{\rm Combine}_{L}\hspace{-0.1em}(\cdot)$ denotes the optimal combine function that iteratively updates literal embeddings by the aid of the optimal message. Eq. \ref{ocl} implies the local Boolean variable flipping in WalkSAT: if the norm of ${m}^{(k)}_{v}$ is the maximum among all the optimal literal messages, its literal embedding would be replaced by the embedding of its negation, otherwise, keep the identical value. The maximization ensures only one literal embedding that would be ``flipped'' per iteration, which simulates the local search behavior. Besides, the literal embedding selected for update would not be $\mathbf{0}$, which implies all the clauses containing this literal are satisfied (see the condition 2 in Eq. \ref{oal} ). Since all the satisfied clauses would not be selected in WalkSAT, this literal also would not be selected to update in this iteration. Finally, if a literal has been included by a clause that is unsatisfied, it would be randomly picked in some probability. The uncertainty is implied by the randomness of $\boldsymbol{\epsilon}^{(k)}$.    
		\begin{small}
			\begin{equation}
			\begin{aligned}\label{oca}
			{m}^{(k)}_{\Psi(v)}\hspace{-0.3em}&=\hspace{-0.3em}\overline{\rm Aggregate}_{C}\big(\{h^{(k-1)}_{u}: u\in\Psi(v)\}\big)\\
			&=\left\{
			\begin{aligned}
			h^{(0)}_{\Psi(v)} &, &{\rm Sigmoid}\Big({\rm MLP^\ast_2\Big(}\sum_{u\in \Psi(v)}{\rm MLP_1}^\ast(h^{(k-1)}_{u}){\Big)}\Big)\geq 0.5  \\
			\boldsymbol{0} &, &{\rm Sigmoid}\Big({\rm MLP^\ast_2\Big(}\sum_{u\in \Psi(v)}{\rm MLP_1}^\ast(h^{(k-1)}_{u}){\Big)}\Big)<0.5 
			\end{aligned}
			\right.\ \ \ \ \ \ s.t. \ \forall \Psi(v)\in \Phi
			\end{aligned}
			\end{equation}\end{small}where $\overline{\rm Aggregate}_{C}(\cdot)$ denotes the optimal aggregation function that conveys the clause embedding messages during reasoning. Note that ${\rm MLP_2\Big(}\sum_{u\in \Psi(v)}{\rm MLP_1}(h^{(k-1)}_{u}){\Big)}$ indicates Deep Sets \citep{zaheer2017deep}, a neural network that encodes a literal embedding set $\{h^{(k-1)}_{u}\}_{u\in\Psi(v)}$ whose literals are included by a clause $\Psi(v)$. The reduced feature would be fed into the sigmoid clause predictor. We use ${\rm MLP_1}^\ast$ and ${\rm MLP}^\ast_2$ to denote the implicit optimal prediction to each clause: given the arbitrarily initiated literal embeddings that denote the Boolean value assignment of literals, the optimal Deep Sets can predict whether the literal-derived clause is satisfied ($\geq 0.5$) or not ($< 0.5$). Since the predictor is permutation-invariant to the input, Propositions 3.1 in \citep{DBLP:journals/corr/abs-1905-13211} promises that it can be approximated arbitrarily closely by graph convolution, which exactly corresponds to the parameterized clause aggregation functions in Eq.2. On the other hand, Eq. \ref{ocl} promises the literal embeddings staying in their initiated values over iterations, hence the optimal Deep Sets may alway judge whether a clause (the set of literals as the input of Deep Sets) is satisfied or not.     
		\begin{small}\begin{equation}\begin{aligned}\label{occ}
			h^{(k)}_{\Psi(v)}\hspace{-0.3em}&=\hspace{-0.3em}\overline{\rm Combine}_{C}\Big(h^{(k-1)}_{\Psi(v)}, m^{(k)}_{\Psi(v)}\Big)
			\\    &=\left\{
			\begin{aligned}
			h^{(k-1)}_{\Psi(v)} &, & h^{(k-1)}_{\Psi(v)}=m^{(k)}_{\Psi(v)} \\
			h^{(0)}_{\Psi(v)} &, &||h^{(k-1)}_{\Psi(v)}||<||m^{(k)}_{\Psi(v)}|| \\
			\boldsymbol{0}&, &||h^{(k-1)}_{\Psi(v)}||\geq||m^{(k)}_{\Psi(v)}||
			\end{aligned}
			\right.\ \ \ \ \ \ s.t. \ \forall \Psi(v)\in \Phi\end{aligned}
			\end{equation}\end{small}where $\overline{\rm Combine}_{C}(\cdot)$ denotes the optimal clause combine function. Based on the propagated messages conveyed by Eq. \ref{bgnn} , it determines how to iteratively update clause embeddings to simulate WalkSAT. 
		
		Here we elaborate how the four optimal functions above cooperate to simulate an iteration of WalkSAT. Since GNNs use literal embeddings as the initial input, we first analyze Eq. \ref{oca} and takes a literal $v$ into our consideration. As we discussed, this function receives a set of literal embeddings that denotes a clause that contains $v$, and then, takes the optimal Deep Sets as an oracle to judge whether this clause is satisfied. The output, the optimal message about the clause, equals to the initiated embedding of the clause $h_{\Psi(v)}$ if it is satisfied, otherwise becomes $\mathbf{0}$. This process simulates the logical reasoning on a clause, which WalkSAT relies on to pick an unsatisfied clause and flip one of its variables (see Eq. \ref{ocl}). Based on $m^{(k)}_{\Psi(v)}$, the optimal clause combine function (Eq. \ref{occ}) updates an arbitrary clause embedding that contains $v$. The first branch states that, if the current clause message $m^{(k)}_{\Psi(v)}$ is consistent with the previous clause embedding $h^{(k-1)}_{\Psi(v)}$, it implies the satisfiability of the clause $\Psi(v)$ is not changed in this iteration (the previously satisfied clause is still satisfied, vice and versa). In this case the clause embedding would not be updated. The second and third branches imply that when $m^{(k)}_{\Psi(v)}$ and $h^{(k-1)}_{\Psi(v)}$ are inconsistent, how to update the clause embedding $h^{(k)}_{\Psi(v)}$ to convey the current message about whether the clause $\Psi(v)$ is satisfied (return into the initial clause embeddings) or not (turn into $\boldsymbol{0}$). Therefore all updated embeddings about the clauses that contain $v$, as the neighbors of $v$, would be fed into the optimal aggregation function in Eq. \ref{oal} . This function selects $v$ that only exists in satisfied clauses, \emph{i.e.}, $\prod_{\Psi(v)}||h^{(k)}_{\Psi(v)}||\neq 0$ (If there is an unsatisfied clauses, its embedding is $\mathbf{0}$ according to Eq. \ref{occ} ,and would lead to $\prod_{\Psi(v)}||h^{(k)}_{\Psi(v)}||= 0$), then the embedding of $v$ would become $\mathbf{0}$. The results by this operation are taken advantage by Eq. \ref{ocl} , which promises the literal that only exists in satisfied clauses would not be ``flipped'' (WalkSAT only chooses unsatisfied clause and select its variables to flip. If literals are not in any unsatisfied clauses, it would not be chosen). Towards the literal $v$ contained by one unsatisfied clause at least ($\prod_{\Psi(v)}||h^{(k)}_{\Psi(v)}||= 0$ since there exists a clause embedding equals to $\mathbf{0}$ according to Eq. \ref{occ} ), its literal message would be assigned by a random vector $\boldsymbol{\epsilon}^{(k)}$. It implies the randomness when WalkSAT try to select one of literal in unsatisfied clauses to flip its value. The flipping process is simulated by Eq. \ref{oca} as we have discussed. 
		
		Here we futher verify if a CNF formula could be satisfied, literal embeddings generated by the optimal aggregation and combine functions that represent the Boolean assignment of literal to satisfy this CNF formula, would converge over iterations (It corresponds to the stop criteria in WalkSAT.). Specifically suppose that in the $k$-$1$ iteration, Eq. \ref{ocl} have induced the literal embeddings so that all clauses with the literal in the formula have been satisfied. By Eq. \ref{oca} it is obvious that $\forall v\in L$, $m^{(k)}_{\Psi(v)}=h^{(0)}_{\Psi(v)}$. To this we have $h_{\Psi(v)}^{(k-1)}\hspace{-0.5em}=\hspace{-0.2em}m^{(k)}_{\Psi(v)}$ and $h_{\Psi(v)}^{(k)}\hspace{-0.5em}=\hspace{-0.2em}h^{(k-1)}_{\Psi(v)}\hspace{-0.5em}=\hspace{-0.5em}h^{(0)}_{\Psi(v)}$ since all clauses in the formula have already been satisfied before the current iteration. In this case, it holds $\prod_{\Psi(v)}||h^{(k-1)}_{\Psi(v)}||\neq 0$ and leads to $\forall v\in L$, $m^{(k)}_v=\mathbf{0}$ in this formula (Eq. \ref{oal}). In term of this, Eq. \ref{ocl} guarantees all the literal embeddings consistent with those in the previous iteration. 
		
		Concluding the analysis above, we know that the optimal aggregation and combine functions (Eq, \ref{oal} \ref{ocl} \ref{oca} \ref{occ} ) are cooperated to simulate the local search in WalkSAT. 
	\end{proof}
	
	\textbf{Failure in 2QBF.} Notably the failure in proving UNSAT would not be a problem for GNNs applied to solve SAT, as
	predicting satisfiability with high confidence has already been good enough for a binary distinction. However, 2QBF problems imply solving UNSAT, which inevitably makes GNNs unavailable in proving the relevant formulae. It probably explains the mystery in{\color{red} \cite{10.1007/978-3-319-40970-2_27}} about why GNNs purely learned by data-driven supervised learning lead to the same performances as random speculation \cite{DBLP:journals/corr/abs-1904-12084}. 
	
	\section{Further Discussion}\label{sec:related work}
	In this manuscript, we provide some discussions about the GNNs that consider the SAT and 2QBF problem as static graph, we haven't considered the shrinkage condition, which may apply dynamic GNN as \cite{DBLP:journals/corr/abs-1902-10191}, dues to the difficulty about proving the dynamic graph as we need to prove all the dynamic updating methods are impossible or not. Ought to be regarded that, this manuscript \emph{does not claim} GNN is provably unable to achieve UNSAT, which remains an open issue. 
	
	Belief propagation (BP) is a Bayesian message-passing method first proposed by \citep{DBLP:conf/aaai/Pearl82},
	which is a useful approximation algorithm and has been applied to the SAT problems (specifically in 3-SAT 
	\citep{SRSP:journals/science/0036-8075}) and 2QBF problems \citep{DBLP:journals/corr/abs-1202-2536}. 
	BP can find the witnesses of unsatisfiability of 2QBF by adopting a bias estimation strategy. 
	Each round of BP allows the user to select the most biased 
	$\forall$-variable and assign the biased value to the variable. 
	After all the $\forall$-variables are assigned, the formula is simplified by the assignment and sent to SAT solvers. 
	The procedure returns the assignment as a witness of unsatisfiability if the simplified formula is unsatisfiable, 
	or UNKNOWN otherwise. However, the fact that BP is used for each $\forall$-variable assignment leads to high overhead, 
	similar to the RL approach given by \citep{DBLP:journals/corr/abs-1807-08058}. 
	It is interesting, however, to see that with the added overhead, BP can find witnesses of unsatisfiability,  which is what one-shot GNN-based embeddings cannot achieve.
	
	This manuscript revealed the previously unrecognized limitation of GNN in reasoning about unsatisfiability of SAT problems. This limitation is probably rooted in the simpility of message-passing scheme, which is good enough for embedding graph features, but not for conducting complex reasoning on top of the graph structures.

	\bibliography{reference}

\begin{thebibliography}{10}

\bibitem{amizadeh2018learning}
Saeed Amizadeh, Sergiy Matusevych, and Markus Weimer.
\newblock Learning to solve circuit-{SAT}: An unsupervised differentiable
  approach.
\newblock In {\em International Conference on Learning Representations}, 2019.

\bibitem{DBLP:journals/jsyml/Church36}
Alonzo Church.
\newblock A note on the entscheidungsproblem.
\newblock {\em J. Symb. Log.}, 1(1):40--41, 1936.

\bibitem{Cook:1971:CTP:800157.805047}
Stephen~A. Cook.
\newblock The complexity of theorem-proving procedures.
\newblock In {\em Proceedings of the Third Annual ACM Symposium on Theory of
  Computing}, STOC '71, pages 151--158, New York, NY, USA, 1971. ACM.

\bibitem{DBLP:journals/cacm/DavisLL62}
Martin Davis, George Logemann, and Donald~W. Loveland.
\newblock A machine program for theorem-proving.
\newblock {\em Commun. {ACM}}, 5(7):394--397, 1962.

\bibitem{DBLP:journals/corr/abs-1807-08058}
Gil Lederman, Markus~N. Rabe, and Sanjit~A. Seshia.
\newblock Learning heuristics for automated reasoning through deep
  reinforcement learning.
\newblock {\em CoRR}, abs/1807.08058, 2018.

\bibitem{DBLP:journals/corr/LiangSFLY16}
Xiaodan Liang, Xiaohui Shen, Jiashi Feng, Liang Lin, and Shuicheng Yan.
\newblock Semantic object parsing with graph {LSTM}.
\newblock {\em CoRR}, abs/1603.07063, 2016.

\bibitem{10.1007/978-3-319-40970-2_27}
Florian Lonsing, Uwe Egly, and Martina Seidl.
\newblock Q-resolution with generalized axioms.
\newblock In Nadia Creignou and Daniel Le~Berre, editors, {\em Theory and
  Applications of Satisfiability Testing -- SAT 2016}, pages 435--452, Cham,
  2016. Springer International Publishing.

\bibitem{SRSP:journals/science/0036-8075}
M.~"M{\'e}zard", G.~Parisi, and R.~Zecchina.
\newblock Analytic and algorithmic solution of random satisfiability problems.
\newblock {\em Science}, 297(5582):812--815, 2002.

\bibitem{DBLP:journals/corr/abs-1902-10191}
Aldo Pareja, Giacomo Domeniconi, Jie Chen, Tengfei Ma, Toyotaro Suzumura,
  Hiroki Kanezashi, Tim Kaler, and Charles~E. Leisersen.
\newblock Evolvegcn: Evolving graph convolutional networks for dynamic graphs.
\newblock {\em CoRR}, abs/1902.10191, 2019.

\bibitem{DBLP:conf/aaai/Pearl82}
Judea Pearl.
\newblock Reverend bayes on inference engines: {A} distributed hierarchical
  approach.
\newblock In {\em {AAAI}}, pages 133--136. {AAAI} Press, 1982.

\bibitem{DBLP:conf/dimacs/SelmanKC93}
Bart Selman, Henry~A. Kautz, and Bram Cohen.
\newblock Local search strategies for satisfiability testing.
\newblock In {\em Cliques, Coloring, and Satisfiability}, volume~26 of {\em
  {DIMACS} Series in Discrete Mathematics and Theoretical Computer Science},
  pages 521--531. {DIMACS/AMS}, 1993.

\bibitem{DBLP:conf/iclr/SelsamLBLMD19}
Daniel Selsam, Matthew Lamm, Benedikt B{\"{u}}nz, Percy Liang, Leonardo
  de~Moura, and David~L. Dill.
\newblock Learning a {SAT} solver from single-bit supervision.
\newblock In {\em {ICLR} (Poster)}. OpenReview.net, 2019.

\bibitem{DBLP:series/faia/SilvaLM09}
Jo{\~{a}}o P.~Marques Silva, In{\^{e}}s Lynce, and Sharad Malik.
\newblock Conflict-driven clause learning {SAT} solvers.
\newblock In {\em Handbook of Satisfiability}, volume 185 of {\em Frontiers in
  Artificial Intelligence and Applications}, pages 131--153. {IOS} Press, 2009.

\bibitem{DBLP:conf/iclr/XuHLJ19}
Keyulu Xu, Weihua Hu, Jure Leskovec, and Stefanie Jegelka.
\newblock How powerful are graph neural networks?
\newblock In {\em {ICLR}}. OpenReview.net, 2019.

\bibitem{DBLP:journals/corr/abs-1905-13211}
Keyulu Xu, Jingling Li, Mozhi Zhang, Simon~S. Du, Ken{-}ichi Kawarabayashi, and
  Stefanie Jegelka.
\newblock What can neural networks reason about?
\newblock {\em CoRR}, abs/1905.13211, 2019.

\bibitem{DBLP:journals/corr/abs-1904-12084}
Zhanfu Yang, Fei Wang, Ziliang Chen, Guannan Wei, and Tiark Rompf.
\newblock Graph neural reasoning for 2-quantified boolean formula solvers.
\newblock {\em CoRR}, abs/1904.12084, 2019.

\bibitem{DBLP:journals/corr/abs-1810-02338}
Kexin Yi, Jiajun Wu, Chuang Gan, Antonio Torralba, Pushmeet Kohli, and
  Joshua~B. Tenenbaum.
\newblock Neural-symbolic {VQA:} disentangling reasoning from vision and
  language understanding.
\newblock {\em CoRR}, abs/1810.02338, 2018.

\bibitem{zaheer2017deep}
Manzil Zaheer, Satwik Kottur, Siamak Ravanbakhsh, Barnabas Poczos, Ruslan~R
  Salakhutdinov, and Alexander~J Smola.
\newblock Deep sets.
\newblock In {\em Advances in neural information processing systems}, pages
  3391--3401, 2017.

\bibitem{DBLP:journals/corr/abs-1202-2536}
Pan Zhang, Abolfazl Ramezanpour, Lenka Zdeborov{\'{a}}, and Riccardo Zecchina.
\newblock Message passing for quantified boolean formulas.
\newblock {\em CoRR}, abs/1202.2536, 2012.

\bibitem{DBLP:journals/corr/abs-1807-09244}
David Zheng, Vinson Luo, Jiajun Wu, and Joshua~B. Tenenbaum.
\newblock Unsupervised learning of latent physical properties using
  perception-prediction networks.
\newblock {\em CoRR}, abs/1807.09244, 2018.

\end{thebibliography}
	\bibliographystyle{plain}

\end{document}